\documentclass{article}

\usepackage{spconf}
\usepackage{amsmath,amssymb,amsfonts,mathtools,bm}
\usepackage{graphicx}
\usepackage{xcolor}
\usepackage{cite}

\usepackage{tabularx}
\usepackage{booktabs}
\usepackage{array}
\newcolumntype{M}[1]{>{\centering\arraybackslash}m{#1}}

\usepackage[linesnumbered,ruled,vlined]{algorithm2e}

\usepackage{subcaption}

\usepackage{textcomp}
\usepackage{comment}

\usepackage{amsthm}

\newtheorem{theorem}{Theorem}
\newtheorem{lemma}{Lemma}

\theoremstyle{definition}
\newtheorem{definition}{Definition}
\newtheorem{remark}{Remark}

\graphicspath{{./images/}}

\usepackage{hyperref}

\title{\lowercase{k}VNN: Learnable Multi-Kernel Volterra Neural Networks}
\name{Haoyu Yun\textsuperscript{1}, Hamid Krim\textsuperscript{1}, Yufang Bao\textsuperscript{2}\thanks{Thanks to the ARO for funding (Grant No.~W911NF2410329).}}
\address{%
\textsuperscript{1}Department of Electrical and Computer Engineering, North Carolina State University, Raleigh, NC, USA\\
\textsuperscript{2}Department of Mathematics and Computer Science, Fayetteville State University, Fayetteville, NC, USA
}

\begin{document}

\maketitle

 \begin{abstract}
Higher-order learning is fundamentally rooted in exploiting compositional features. It clearly hinges on enriching the representation by more elaborate interactions of the data which, in turn, tends to increase the model complexity of conventional large-scale deep learning models. In this paper, a kernelized Volterra Neural Network (kVNN) is proposed. The key to the achieved efficiency lies in using a learnable multi-kernel representation, where different interaction orders are modeled by distinct polynomial-kernel components with compact, learnable centers, yielding an order-adaptive parameterization. Features are learned by the composition of layers, each of which consists of parallel branches of  different polynomial orders, enabling kVNN filters to directly replace standard convolutional kernels within existing architectures. The theoretical results are substantiated by experiments on two representative tasks: video action recognition and image denoising. The results demonstrate favorable performance-efficiency trade-offs: kVNN consistently yields reduced model (parameters) and computational (GFLOPs) complexity with competitive and often improved performance. These results are maintained even when trained from scratch without large-scale pretraining. In summary, we substantiate that structured kernelized higher-order layers offer a practical path to balancing expressivity and computational cost in modern deep networks.

\end{abstract}

\section{Introduction}
Deep neural networks have achieved remarkable performance on a wide range of vision and video tasks, yet their core building blocks remain largely linear (e.g., convolutions or linear projections), with nonlinearity primarily introduced through pointwise activations and depth. Higher-order interactions in data, which explicitly capture multiplicative couplings among features and enable richer nonlinear structures, are typically synthesized only across layers, thus increasing model complexity. Deploying higher-order modeling as a reusable and replaceable layer in large-scale deep models is, however, still challenging, since the parameter count and computational cost often grow rapidly with the interaction order and the input dimension \cite{Roheda_Krim_2020}.

Kernel methods, well-known in signal processing and statistics, provide a principled way to capture higher-order interactions: by representing polynomial-type interactions implicitly in a reproducing kernel Hilbert space (RKHS), one obviates explicit construction of high-dimensional monomial features while retaining a clear functional interpretation \cite{6796712}. Nevertheless, classical kernel regression typically relies on sample-based expansions, whose complexity scales with the number of training points, leading to substantial storage and computation in modern large-scale learning. While many scalable approximations have been developed~\cite{JMLR:v6:drineas05a}, effective integration of kernelized higher-order representations into deep networks via a modular, end-to-end trainable, and budget-controlled manner remains elusive.

In this paper, we propose a kernelized higher-order filtering layer tailored for deep models, translating the expressivity of higher-order (Volterra-type) interactions into a practical network module. The key idea is to characterize the underlying high-order basis functions through a multi-kernel representation: different interaction orders are modeled by different kernel components and are parameterized in a structured, learnable way, thereby balancing expressivity and complexity. The resulting layer is designed as a plug-and-play module, enabling straightforward replacement of standard convolutional filters for end-to-end training within existing architectures.

The main contributions herein are threefold. (i) We introduce a multi-kernel representation for higher-order polynomial (Volterra-type) filtering, providing a structured functional basis that separates interaction orders. (ii) We instantiate this representation as a practical and modular network layer that can directly replace standard convolutional filters and allows end-to-end training. (iii) We validate the proposed design on two representative tasks---video action recognition and image denoising---demonstrating favorable accuracy--efficiency trade-offs under controlled complexity.

The rest of the paper is organized as follows:  Section~\ref{background} reviews background on higher-order modeling, the kernelization perspective, and higher-order interactions in deep networks. Section~\ref{method} presents our proposed methodology, including the learnable multi-kernel representation and the kVNN layer construction. Section~\ref{experiments} reports experimental results on video action recognition and image denoising. Section~\ref{conclusion} provides conclusions and an outlook on future directions.

\section{Background}
\label{background}

\subsection{Volterra Filtering}
\label{sec:bg_volterra}

A nonlinear Volterra system response extends linear convolution to include nonlinear terms, resulting in the continuous time $p$-th order Volterra series  defined as:
\begin{equation}
\label{eq:volterraconv}
f(x) =H_0x(t) +H_1x(t) +H_2x(t)+\cdots +H_px(t),
\end{equation}
where $H_0x(t)=const.$ (it is set as $0$ by default), and
\begin{equation*}
\label{eq:volterra_ser}
H_r x(t)  \!= \!\!\int!\!\! h_{r} (\!s_1,\! \cdots\!,\! s_r) x(t\!-\!s_1)\cdots x(t\!-\!s_r) ds_1\!\cdots \! ds_r.
\end{equation*}
In this work, we use the standard truncated discrete/vectorized form. Let $x \in \mathbb{R}^{d}$ denote a finite input of $x(t)$  (e.g., a vectorized signal or an image patch). Discretizing equation~\eqref{eq:volterraconv} yields $p$-th order  Volterra mapping:
\begin{equation}
\label{eq:volterra_trunc_vec}
\begin{aligned}
f(x)
&= \sum_{r=1}^{p} f_r(x),\\
f_r(x)
&= \sum_{i_1=1}^{d}\cdots\sum_{i_r=1}^{d}
h_r(i_1,\ldots,i_r)\,x_{i_1}\cdots x_{i_r}.
\end{aligned}
\end{equation}
Here, $h_r$ is an $r$-order coefficient tensor that captures $r$-way interactions among input coordinates. This formulation naturally generalizes linear filtering ($r=1$) by explicitly modeling higher-order correlations via $r\ge 2$ terms. Such higher-order filtering has recently been revisited in modern deep learning via Volterra Neural Networks~\cite{Roheda_Krim_2020} to develop principled higher-order neural layers and discuss their representation and complexity properties.

A recognized challenge of Volterra filtering is the rapidly increasing complexity: the number of free coefficients in $h_r$ scales as $\mathcal{O}(d^r)$ in the general case (and remains polynomially large even under symmetry/structure constraints). As a result, directly using high-order Volterra filters in deep architectures can be computationally demanding when $d$ or the interaction order $r$ grows. Structured parameterizations that preserve the expressive benefits of higher-order interactions while controlling parameter and computational costs are clearly of interest and constitute the central goal of the following development.

\subsection{Kernelization View of Volterra Modeling}
\label{subsec:kernelization_background}
As noted, seeking  high-order Volterra-series-based modeling entails quickly  increased model and computational complexities  for increasing input dimensions. Kernel methods, well known for their efficiency, are proposed as a path towards alleviating this difficulty by implicitly representing polynomial interactions in a reproducing kernel Hilbert space (RKHS), and can indeed avoid an explicit enumeration of high-order coefficients~\cite{6796712}. In particular, for nonlinear mappings that can be expressed (or well-approximated) by truncated polynomials, a polynomial kernel induces an RKHS in which these nonlinear effects correspond to linear functionals.

Formally, let $K:\mathbb{R}^d\times\mathbb{R}^d\to\mathbb{R}$ be a positive semidefinite kernel.
Then $K$ admits an inner-product feature representation (Theorem~\ref{thm:map}, Appendix):
\begin{equation}
\label{eq:rkhs_featuremap_bg}
K(x,x')=\langle \phi(x),\phi(x')\rangle_{\mathcal{F}}.
\end{equation}
To capture polynomial interactions, a standard choice is the degree-$r$ homogeneous polynomial kernel
\begin{equation}
\label{eq:poly_kernel_bg}
K_r(x,x')\triangleq (x^\top x')^{r}.
\end{equation}

The RKHS induced by \(K_r(\,\cdot\,,\,\cdot\,)\) includes (finite-dimensional) degree-\(r\) homogeneous polynomials. Consequently, a truncated polynomial nonlinearity of order up to \(p\) can be interpreted as a function in a union of order-specific polynomial subspaces, where each order \(r\) corresponds to a distinct subspace. This RKHS perspective is particularly useful for Volterra modeling because it connects high-order interactions to linear operations in feature space: learning can be pursued via inner products \(x^\top x'\) (or kernel evaluations \(K_r(x,x')\)) without explicitly forming monomial features.

In classical kernel regression, the representer theorem~\cite{10.5555/648300.755324} further states that solutions of many regularized empirical risk minimization problems admit a finite expansion over the training samples. Concretely, for a broad class of objective functions, the learned function can be written as
\begin{equation}
\label{eq:representer_bg}
f(x)=\sum_{j=1}^{N}\gamma_j\,K(x,x_j),\qquad \gamma_j\in\mathbb{R},
\end{equation}
where \(\{x_j\}_{j=1}^{N}\) are the training inputs. Equation~\eqref{eq:representer_bg} highlights that kernel learning can be implemented through kernel evaluations and a set of coefficients \(\{\gamma_j\}\), providing a convenient computational form for estimating nonlinear mappings.

Despite these advantages, directly adopting classical kernel regression formulations in modern deep learning pipelines remains nontrivial. The sample-based expansion in Equation~\eqref{eq:representer_bg} scales with the number of training points \(N\), leading to unreasonably large storage and computational costs in large-scale settings; accordingly, sparser or lower-complexity approximations such as Nystr\"om/column-subsampling methods~\cite{NIPS2015_f3f27a32,10.1016/j.inffus.2015.03.001}, random feature embeddings~\cite{JMLR:v21:19-1031}, and scalable training schemes for large-scale kernel learning~\cite{NIPS2014_c6cc81e8} have been developed, yet their integration into deep networks still faces approximation--accuracy trade-offs and limited flexibility under fixed budgets. Moreover, while polynomial kernels provide a principled way to model high-order interactions, deep networks typically favor modular and parameter-efficient designs. In contrast to commonly adopted local architectural fine-tuning and partial mitigations to pursue improvement, we revisit the non-linear approximation in learning, by way of a multiple-kernel data representation. In this context, multiple kernel learning (MKL) provides a complementary direction by combining base kernels and learning their relative contributions via particularly relevant feature interaction patterns. Note classical MKL methods~\cite{8627941,pmlr-v119-ghari20a} typically lack the structured, layer-wise formulation required for efficient, end-to-end deep representation learning.

\subsection{Higher-order Modeling in Deep Neural Networks}

In deep learning, a number of studies have explored incorporating higher-order interactions or higher-order filtering structures into network layers to enhance feature representations and capture richer nonlinear relationships. A representative example is the second-order/bilinear formulation: Bilinear CNN computes outer products between two feature streams followed by pooling, thereby explicitly modeling local second-order interactions and achieving strong performance in fine-grained recognition \cite{7410527}. However, such second-order representations often lead to extremely high-dimensional features (the outer-product dimension grows quadratically with the channel dimension), resulting in a significant increase in storage and in computational costs. To address this issue, subsequent work proposed compact bilinear pooling, which employs approximation techniques to compress the representation into a tractable dimension while largely preserving discriminative power, thereby alleviating the overhead of high-dimensional bilinear features \cite{7780410}.

Beyond explicit second-order constructions, other works have adopted a more general view of higher-order polynomial networks/polynomial layers, treating the network output as a high-order polynomial of the input and using tensor factorization or related techniques to reduce the parameter size, so that higher-order structures can be implemented and trained in a factorized form \cite{9156685}. Closely related is the direction of inserting higher-order tensor mappings as trainable layers. For instance, Tensor Regression Networks employ low-rank multilinear mappings to project high-order activation tensors into the output space, yielding a more compact parameterization and more controllable computational complexity \cite{10.5555/3455716.3455839}. From the perspective of filters and convolutional operators, higher-order convolutional designs inspired by Volterra expansions have been proposed to model higher-order nonlinearities and structural properties such as equivariance \cite{9247263}. Related cascaded filtering constructions have also been explored for image classification \cite{GHANEM2023200210} and higher-order interaction modeling \cite{JMLR:v25:21-1082}.

Higher-order interaction layers are highly expressive, but their use in large deep models is often limited by rapidly increasing model and compute costs with interaction order and input dimension. Many existing formulations function as head-level aggregation or task-specific mapping modules, so their effectiveness depends on task settings and head design. In contrast, our goal is to replace standard convolutions with the proposed method at the operator level under matched topology and training, making these approaches not directly comparable under a controlled evaluation.


\section{Methodology}
\label{method}

Building on the demonstrated learning capacity of Volterra Filtering and the potential of kernelization in Section~\ref{background}, we proceed with the key idea of characterizing Volterra basis functions. Adopting a multi-kernel representation, i.e.\ a weighted linear combination of multiple kernel functions, is shown to yield a more efficient realization of a Volterra filter with a correspondingly flexible and more compact parametrization of higher order interaction terms. To that end, we first introduce the Multi-Kernel Volterra representation and its theoretical justification to subsequently provide a constructive procedure to instantiate kVNN kernels corresponding to the second- and third-order components. We note that the current approach can be extended to arbitrary order with manageable growth in complexity.

\subsection{Learnable multi-kernel representation}

The RKHS and representer-theorem background in Section~\ref{background} offers a direct sample-centered kernel expansion, which scales with the dataset size and is hence ill-suited for integrating high-order Volterra interactions into compositional layers of learning networks. Our approach pursues a compact kernelized parameterization in which the kernel functions are not tied to all training samples and can be organized in an order-adaptive manner.

To that end, we introduce a learnable-kernel formulation, with an order-wise expansion which captures multi-order interactions:
\begin{equation}
\label{eq:learnable_kernel_expansion}
k(x,w)=\sum_{r=1}^{p}\beta_r\,(x^\top w)^r,
\end{equation}
where $x\in\mathbb{R}^{d}$ is the input, $w\in\mathbb{R}^{d}$ denotes a kernel center, and $\{\beta_r\}_{r=1}^{p}$ are order-dependent coefficients. This is directly inspired by the geometry of the data manifold characterized by its reach and curvature, yielding a local approximation in contrast to the global approximation strategy of the representer theorem. This strategy not only reduces the approximation to a finite set of $M$ charts within a manifold atlas but also minimizes unnecessary and complexity inducing cross terms. Instead of choosing $w$ from the training samples, we treat the centers as learnable parameters forming a compact set of centers $\{w_i\}_{i=1}^{M}\subset\mathbb{R}^{d}$ with $M\ll N$. The resulting mapping is represented by
\begin{equation}
\label{eq:compact_representation}
f(x)\approx \sum_{i=1}^{M}\gamma_i\,k(x,w_i),
\end{equation}
where $\{\gamma_i\}_{i=1}^{M}$ are learnable coefficients. As the centers are learned in a data-driven manner, this representation aggregates information in the data space with far fewer kernel functions around each chart and across the atlas cover of the manifold, while significantly reducing both computational and storage complexity. Crucially, the reduced number of basis functions makes it possible to allocate dedicated kernel components to different interaction orders without incurring prohibitive complexity.

In light of the learnable-kernel formulation above, we now introduce a multi-kernel construction that explicitly captures the multi-order structure of Volterra models. A $p$-th order truncated Volterra mapping can naturally be decomposed into a sum of order-specific components. For a vectorized input $x\in\mathbb{R}^{d}$ with entries $\{x_i\}_{i=1}^{d}$, we write
\begin{equation}
\label{eq:volterra_decomp}
f(x)=\sum_{r=1}^{p} f_r(x),
\end{equation}
where the $r$-th order component is a homogeneous polynomial of degree $r$, defined as:
\begin{equation}
\label{eq:volterra_fr}
f_r(x)=\sum_{i_1=1}^{d}\cdots\sum_{i_r=1}^{d} h_r(i_1,\ldots,i_r)\,x_{i_1}\cdots x_{i_r}.
\end{equation}
This order-wise decomposition suggests that different interaction orders may exhibit distinct characteristics of the data manifold, and it is therefore natural to model each order using a specialized kernel component. Note that this also reduces redundant cross terms.

To this end, we define an order-$r$ kernel basis function (degree-$r$ kernel term) as
\begin{equation}
\label{eq:kr_def}
k_r(x,w)\triangleq (x^\top w)^r,
\end{equation}
so that the multi-order kernel in~\eqref{eq:learnable_kernel_expansion} can be viewed as a weighted sum of these order-$r$ terms. For flexibility, we forego a single shared center set for all orders, and instead introduce order-dependent learnable centers. Specifically, for each order $r\in\{1,\ldots,p\}$, we learn a compact set of centers
\begin{equation}
\label{eq:center_set}
\mathcal{W}_r=\{w_{r,i}\}_{i=1}^{M_r}\subset\mathbb{R}^{d},\qquad M_r\ll N,
\end{equation}
together with coefficients $\{\gamma_{r,i}\}_{i=1}^{M_r}$. The $r$-th order component is then approximated by a multi-center expansion
\begin{equation}
\label{eq:fr_expand}
f_r(x)\approx \sum_{i=1}^{M_r}\gamma_{r,i}\,k_r(x,w_{r,i})
=\sum_{i=1}^{M_r}\gamma_{r,i}\,(x^\top w_{r,i})^r.
\end{equation}
Combining all orders yields the proposed learnable multi-kernel Volterra representation
\begin{equation}
\label{eq:mk_volterra}
f(x)\approx \sum_{r=1}^{p}\sum_{i=1}^{M_r}\gamma_{r,i}\,(x^\top w_{r,i})^r.
\end{equation}
Equation~\eqref{eq:mk_volterra} explicitly highlights that the overall basis can be written as a sum of different kernel functions: each interaction order $r$ is modeled by its own degree-$r$ kernel term $k_r$ and its own learned center set $\mathcal{W}_r$. This order-specific design increases modeling flexibility for multi-order interactions while keeping the overall complexity controlled through the compactness constraint $M_r\ll N$.

From a functional-space perspective, each order $r$ can be associated with a dedicated RKHS $\mathcal{H}_r$ induced by the degree-$r$ polynomial kernel
\begin{equation}
\label{eq:poly_kernel}
K_r(x,x')\triangleq (x^\top x')^r.
\end{equation}
Accordingly, the overall model can be interpreted as a sum of components in distinct spaces,
\begin{equation}
\label{eq:direct_sum_rkhs}
f \in \mathcal{H}_1 \oplus \mathcal{H}_2 \oplus \cdots \oplus \mathcal{H}_p,
\qquad f=\sum_{r=1}^{p} f_r,\ \ f_r\in\mathcal{H}_r,
\end{equation}
which provides a principled rationale for using specialized kernels and centers for different interaction orders.

For completeness, we provide key theoretical properties of the proposed construction in Appendix~\ref{app:theory_mk}:
Theorem~\ref{lem:mk_compact_nostep} establishes the validity of the weighted multi-kernel and the induced order-wise decomposition,
and Theorem~\ref{thm:finite_atomic_rep} provides a finite atomic expressivity result for truncated Volterra mappings in polynomial linear spaces.

In summary, the proposed learnable multi-kernel formulation provides an order-adaptive and compact parameterization of truncated Volterra mappings by assigning distinct learnable center sets to different interaction orders. This representation serves as the functional basis of our kVNN layer. In the next subsection, we instantiate the above expansion into a practical network module by specifying how the atoms are organized into branches and implemented as convolutional kernels, as illustrated in Fig.~\ref{fig:kvnn_structure}.

\subsection{Layer Structure: kVNN Filter Construction}
\label{subsec:structure}
To proceed with the construction of the proposed kVNN layer, we aggregate multiple branches as the output, each of which generates one interaction order for an input vector $x \in \mathbb{R}^d$, shown in Fig.~\ref{fig:kvnn_structure}. Specifically, an order-$r$ branch is parameterized by learnable centers (projection vectors) and scalar coefficients: the branch first computes a linear projection $x^\top w_{r,i}$ (optimal hyperplane) and then applies an order-specific polynomial nonlinearity by raising the projection to the $r$-th power, thereby enabling interaction modeling.

For a second-order configuration (left panel of Fig.~\ref{fig:kvnn_structure}), the layer consists of combined order-1 linear term and an order-2 term. The order-1 branch only uses a single linear atom, i.e., one learnable projection vector $w_{1}$ together with a coefficient $\gamma_{1}$, providing the basic linear response. The order-2 branch aggregates $n$ quadratic atoms, parameterized by a set of learnable centers $\{w_{2,i}\}_{i=1}^{n}$ and coefficients $\{\gamma_{2,i}\}_{i=1}^{n}$. The final output is the sum of the two contributions, i.e., $y(x)=y_1(x)+y_2(x)$.

For a third-order configuration (right panel of Fig.~\ref{fig:kvnn_structure}), a cubic branch is further added on top of the second-order structure. This order-3 branch aggregates $m$ cubic atoms, parameterized by centers $\{w_{3,i}\}_{i=1}^{m}$ and coefficients $\{\gamma_{3,i}\}_{i=1}^{m}$, leading to $y(x)=y_1(x)+y_2(x)+y_3(x)$.

Empirically, we configure $n$ according to the data complexity; specifically, we set $n=1$ for image-related tasks and increase it to $n=4$ for video-related tasks to ensure sufficient representational power.

A key design choice is that different orders use different center sets (e.g., $\{w_{2,i}\}$ and $\{w_{3,i}\}$ are learned independently), which allows each interaction order to capture distinct characteristics of the input and avoids the limited expressivity induced by sharing centers across orders. Here, $n$ and $m$ are structural hyperparameters that can be set according to the complexity of the data and the task at hand: more complex data distributions or stronger high-order expressivity typically benefit from larger values, whereas simpler tasks or resource-constrained settings may use smaller values to reduce computational and storage overhead. In implementation, the above order-wise filter corresponds to one convolutional kernel in a CNN layer (i.e., a building block that generates one output channel); a convolutional layer typically contains multiple such filters computed in parallel across channels to form the layer output. All centers and coefficients are learned end-to-end with a task loss, and the resulting layer can be viewed as an efficient and structured approximation of a truncated Volterra filter using a finite set of learnable kernel atoms.

\begin{figure}[t]
  \centering
  \includegraphics[width=1\linewidth]{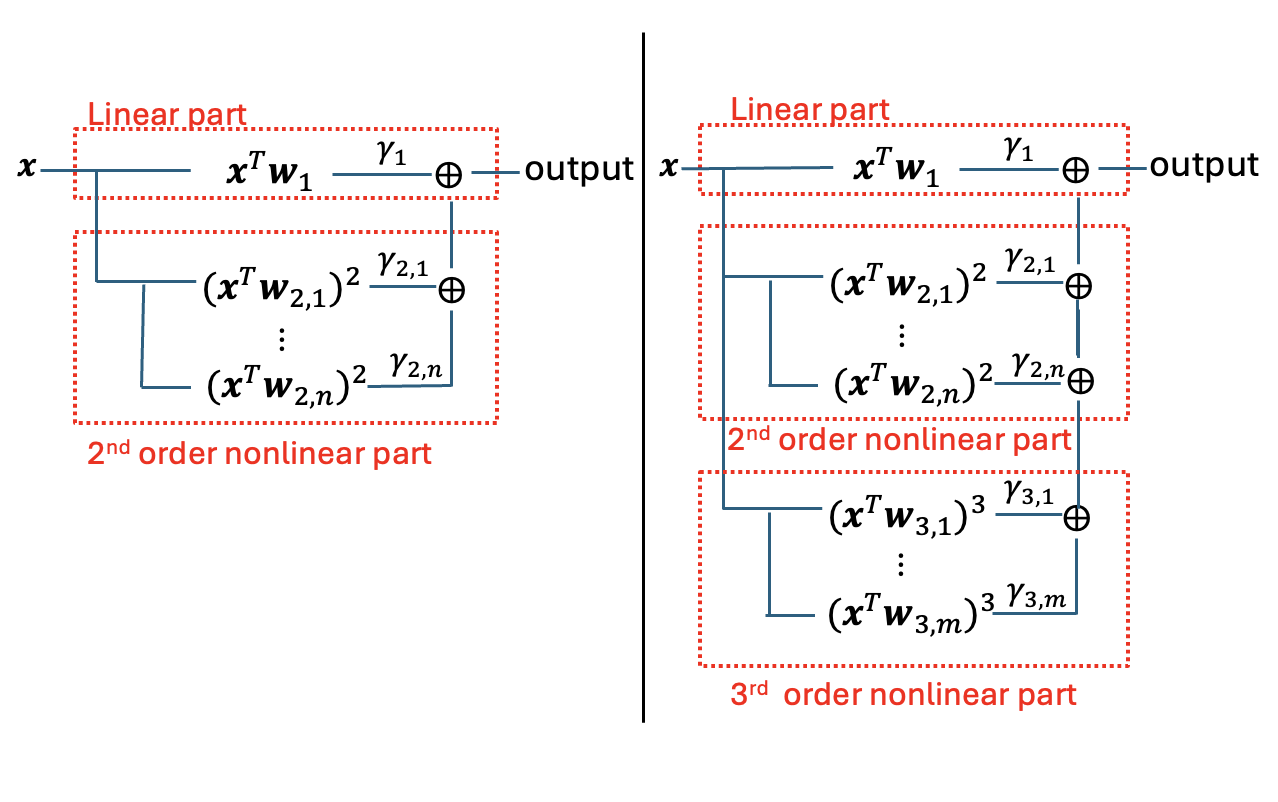}
    \caption{kVNN filter structures. Left: second-order filter composed of an order-1 (linear) branch and an order-2 branch that aggregates $n$ quadratic atoms. Right: third-order filter with an additional order-3 branch that aggregates $m$ cubic atoms. Each order-$r$ branch uses learnable centers $\{w_{r,i}\}$ and coefficients $\{\gamma_{r,i}\}$, and sums atoms of the form $(x^\top w_{r,i})^r$. The order-1 branch uses a single linear atom, and the overall output is obtained by summing the branch outputs across orders.}
    
  \label{fig:kvnn_structure}
\end{figure}

\section{Experiments}
\label{experiments}
To evaluate the effectiveness of the proposed kVNN, we conducted experiments on two representative tasks: video action recognition and image denoising. The following two subsections describe the experimental setups and implementation details for each task, respectively. 

\subsection{Video action recognition task setting}

For video action recognition, we follow a standard two-stream evaluation pipeline (Fig.~\ref{fig:Model_Model_VideoActionRecognition}) to assess the representation capability and computational efficiency of the proposed operator. 
Without introducing any additional specialized modules, RGB frames $X^{\text{rgb}}$ and optical flow fields $X^{\text{flow}}$ are processed by two kVNN-based encoders, an RGB encoder $E_{\text{rgb}}(\cdot)$ and a flow encoder $E_{\text{flow}}(\cdot)$, to yield:

\begin{equation}
    Z^{\text{rgb}} = E_{\text{rgb}}(X^{\text{rgb}}), \quad
    Z^{\text{flow}} = E_{\text{flow}}(X^{\text{flow}}).
\end{equation}
The RGB stream mainly captures appearance information, while the optical-flow stream emphasizes motion patterns between adjacent frames. The encoded features are then fused using a conventional strategy (e.g., concatenation) and fed to a linear classifier to produce the predicted action scores $\hat{y}\in\mathbb{R}^C$.

We consider two capacity-scaled instantiations under the same two-stream pipeline, denoted as kVNN-S and kVNN-B, which only differ in the encoder depth (4 vs. 5 kVNN blocks per stream) to study accuracy--efficiency trade-offs. For each instantiation, we evaluate two kVNN orders: second-order ($p=2$) and third-order ($p=3$). To mitigate overfitting, we apply layer-wise weight decay to kVNN parameters, using stronger decay for deeper layers.

\begin{figure}[t]
    \centering
    \includegraphics[width=1\linewidth]{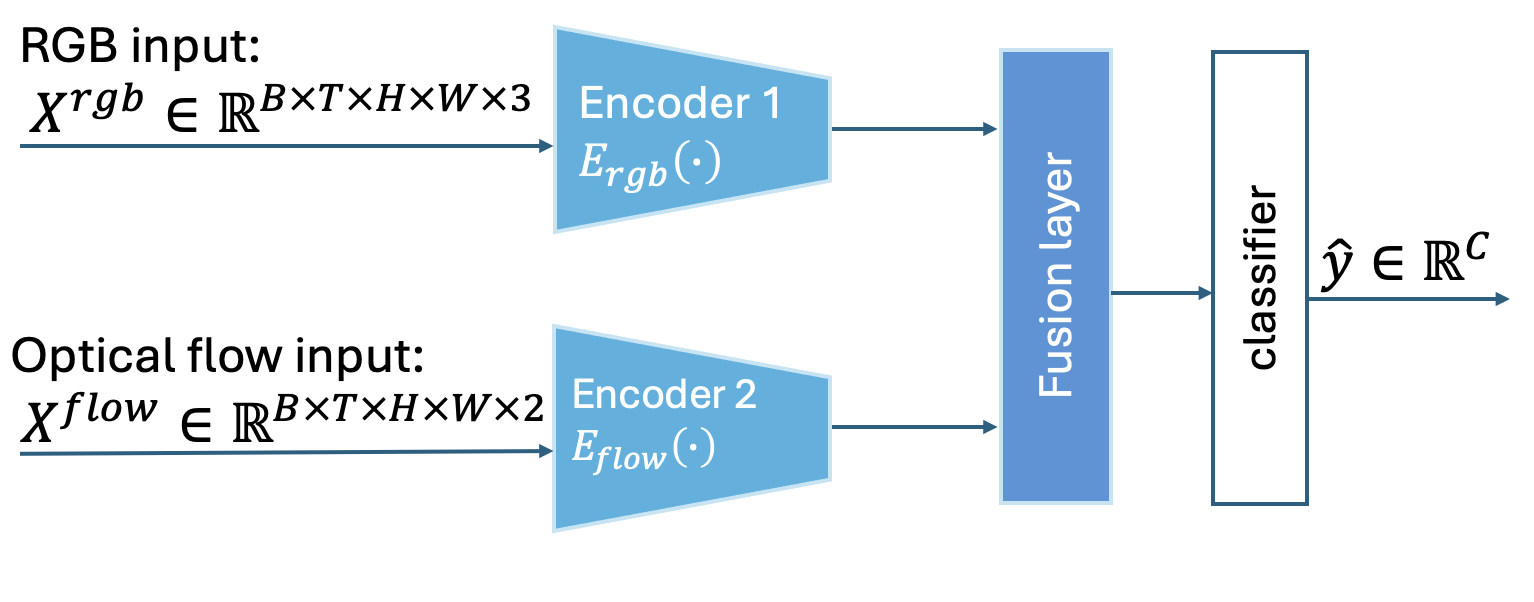}
    \caption{Schematic of a standard two-stream evaluation pipeline instantiated with kVNN-based encoders. RGB frames $X^{\text{rgb}}$ and optical flow $X^{\text{flow}}$ are encoded into $Z^{\text{rgb}}$ and $Z^{\text{flow}}$, fused by a conventional strategy (e.g., concatenation), and passed to a classifier to produce the predicted action scores $\hat{y} \in \mathbb{R}^C$ ($C$ denotes the number of classes).}
    
    \label{fig:Model_Model_VideoActionRecognition}
\end{figure}

We evaluate this two-stream instantiation on two benchmarks, UCF101\cite{DBLP:journals/corr/abs-1212-0402} and HMDB51\cite{6126543}. Results on UCF101 are reported in Table~\ref{tab:ucf101_comparison}, and the joint comparison on both datasets is given in Table~\ref{tab:ucf101_hmdb51_comparison}.

UCF101 contains 13{,}320 video clips from 101 action 
categories, covering a wide range of daily-life and sports scenarios, whereas 
HMDB51 consists of 6{,}766 video clips from 51 categories with more complex 
backgrounds and more ambiguous motions.

\newcommand{\best}[1]{\textcolor{red}{\textbf{#1}}}
\newcommand{\secondbest}[1]{\textcolor{blue}{\textbf{#1}}}

\begin{table}[t]
    \centering
    \scriptsize
    \setlength{\tabcolsep}{3pt}
    \resizebox{\columnwidth}{!}{%
    \begin{tabular}{lccc}
        \toprule
        Model & Param & GFLOPs & ACC (UCF101) \\
        \midrule
        kVNN-S (2nd order filter)  & 7.7M  & \best{15.96} & 86.51\% \\
        kVNN-S (3rd order filter)  & 12.29M & 28.14       & 90.02\% \\
        \midrule
        kVNN-B (2nd order filter)  & 17M   & \secondbest{19.86} & \secondbest{91.17\%} \\
        kVNN-B (3rd order filter)  & 30M   & 35.18              & \best{92.67\%} \\
        \midrule
        I3D~\cite{8099985}                  & 56M  & 63.34 & 88.8\%  \\
        DTMV-CNN~\cite{8249882}           & --   & 83.4  & 87.5\%  \\
        Terao~\cite{10096477}                & --   & 33.6  & 89.2\%  \\
        Video-FocalNet tiny*~\cite{10376718} & 49M  & 63   & 90.1\% \\
        Video-FocalNet*~\cite{10376718}      & 157M & --   & 91.1\% \\
        DVFL-Net*~\cite{11185187}          & 22M  & 27    & 88.4\% \\
        \bottomrule
    \end{tabular}
    }
    \caption{Comparison with existing lightweight models on UCF101 in terms of
    parameter count, GFLOPs, and top-1 accuracy. 
    ``2nd order filter'' means each filter in the model is instantiated as a
    second-order Volterra filter, and ``3rd order filter'' means each filter is
    a third-order Volterra filter. 
    *: results obtained with large-scale pretraining (Kinetics).
    ``--'': model not open-sourced; the corresponding values are unavailable in
    the original paper. 
    Blue entries indicate the second-best result / second-lowest GFLOPs, 
    and red entries indicate the best result / lowest GFLOPs.}
    \label{tab:ucf101_comparison}
\end{table}

Experimental results on UCF101 are summarized in Table~\ref{tab:ucf101_comparison}, which compares four
kVNN variants with several existing lightweight CNN-based video action
recognition models in terms of parameter count, GFLOPs, and top-1 accuracy.
The table also indicates which methods rely on large-scale pretraining (e.g., Kinetics, marked by *), whereas all kVNN variants are trained from scratch on UCF101.

From Table~\ref{tab:ucf101_comparison} it can be observed that, in the lightweight configuration,
kVNN-S (3rd order filter) achieves an accuracy of 90.02\% on UCF101 with only
12.29M parameters and 28.14 GFLOPs, clearly outperforming existing classical
and recent CNN-based methods. Its accuracy is comparable to, or even slightly
better than, that of models pretrained on large-scale datasets, while requiring
significantly fewer parameters and competitive computational cost. The extremely
efficient kVNN-S (2nd order filter) requires only 7.7M parameters and
15.96 GFLOPs, yet still attains an accuracy comparable to other CNN-based
models, a cost-effective option for resource-constrained
scenarios.

In the higher-capacity configuration, kVNN-B (2nd order filter) reaches an
accuracy of 91.17\% with only 17M parameters and 19.86 GFLOPs, achieving one of the best accuracies and the second-lowest computational cost among all methods.
When each filter is further upgraded to a third-order Volterra filter,
kVNN-B (3rd order filter) raises the UCF101 accuracy to 92.67\%, achieving the
best performance among all compared methods, including CNN models that
rely on Kinetics pretraining.

To further illustrate the trade-off between accuracy and computational cost,
the comparison in Table~\ref{tab:ucf101_comparison} is visualized in
Fig.~\ref{fig:ucf101_table_vis}.

\begin{figure}[t]
    \centering
    \includegraphics[width=\linewidth]{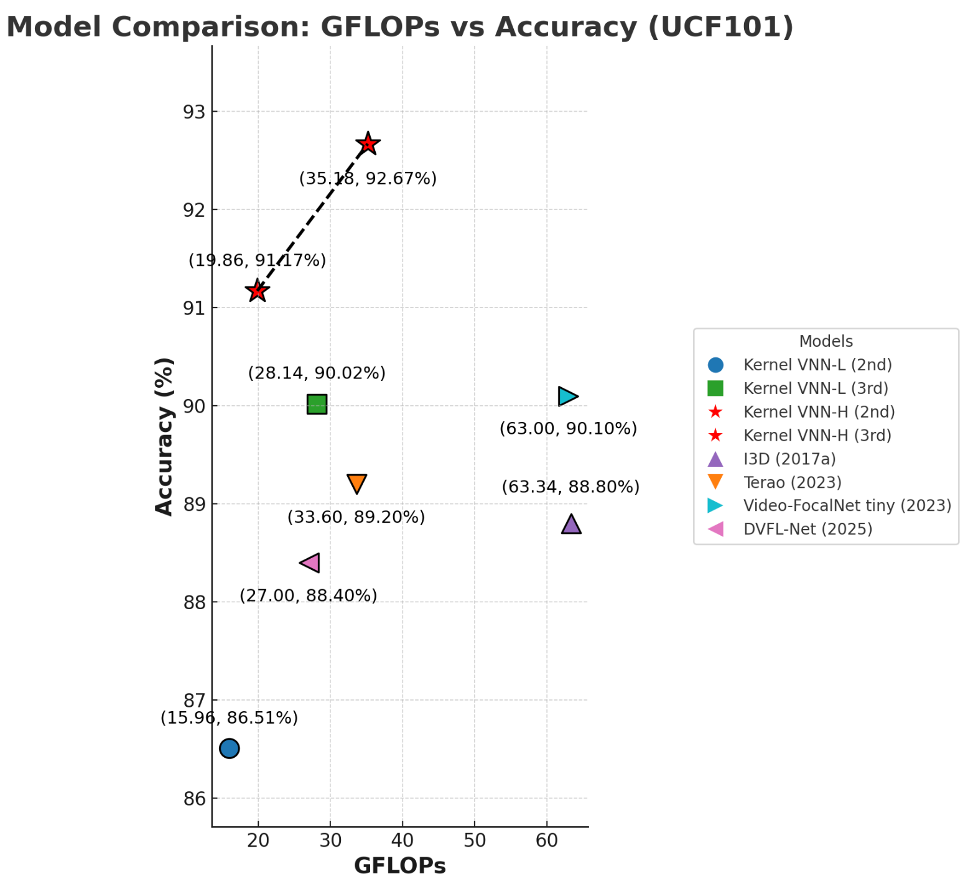}
    \caption{Visualization of the comparison in 
    Table~\ref{tab:ucf101_comparison}. Each point corresponds to a model, 
    showing the trade-off between top-1 accuracy on UCF101 and computational 
    cost (GFLOPs).}
    \label{fig:ucf101_table_vis}
\end{figure}

To further assess generalization across datasets, Table~\ref{tab:ucf101_hmdb51_comparison} reports a broader comparison on both UCF101 and HMDB51. The models built with the kVNN operator consistently outperform classical baselines such as C3D and I3D on both benchmarks. In particular, kVNN-B (3rd order filter) achieves the best accuracy on UCF101 (92.67\%) and obtains strong performance on HMDB51 (67.12\%).
Meanwhile, several competing methods in Table~\ref{tab:ucf101_hmdb51_comparison} rely on large-scale Kinetics pretraining (marked by *), while all kVNN variants are trained from scratch on the target datasets, yet achieve strong HMDB51 accuracy, close to the best reported results that rely on large-scale pretraining.

\begin{table}[t]
    \centering
    \scriptsize
    \setlength{\tabcolsep}{5pt}
    \resizebox{\columnwidth}{!}{%
    \begin{tabular}{lcc}
        \toprule
        Model & ACC (UCF101) & ACC (HMDB51) \\
        \midrule
        kVNN-S (3rd order filter)  & 90.02\% & 64.95\% \\
        kVNN-B (2nd order filter)  & 91.17\% & 65.55\% \\
        kVNN-B (3rd order filter)  & 92.67\% & 67.12\% \\
        \midrule
        C3D~\cite{Tran_2015_ICCV}                     & 85.2\% & 51.6\% \\
        I3D~\cite{8099985}                     & 88.8\% & 62.2\% \\
        MV2flow~\cite{10.1145/3422360}*            & 92.6\% & 65.5\% \\
        Faster-CoViAR~\cite{Xiong2021FasterFCoViARFF}* & 91.2\% & 63.2\% \\
        MVR-AR~\cite{10.1007/978-3-030-68763-2_53}*               & 92.1\% & 60.2\% \\
        IMR-Net~\cite{9414534}*             & 92.6\% & 67.8\% \\
        IMR-Net (scratch)~\cite{9414534}    & 80.2\% & 55.9\% \\
        SOR-TC~\cite{10.1016/j.neucom.2023.02.045}                & 91.0\% & 60.0\% \\
        \bottomrule
    \end{tabular}
    } 
    \caption{Comparison on UCF101 and HMDB51 in terms of top-1 accuracy.
    ``2nd/3rd order filter'' means each filter in the model is instantiated as a
    second-/third-order Volterra filter.
    *: results obtained with large-scale pretraining (Kinetics).}
    \label{tab:ucf101_hmdb51_comparison}
\end{table}

Overall, the kVNN family offers flexible accuracy--complexity trade-offs across different orders and capacities: second-order configurations are more efficient in parameters and computation, making them suitable for resource-limited environments, while third-order configurations significantly improve recognition accuracy at a moderate increase in complexity, even surpassing many pretrained methods. These results demonstrate that the kVNN design based on high-order Volterra filters can achieve a favorable balance between efficiency and performance, even when trained from scratch without any pretraining.

\subsection{Image denoising task setting}

To further demonstrate the efficiency and robustness of the proposed kVNN layer beyond video action recognition, image denoising is used as an additional testbed.
For each CNN baseline, we construct kVNN counterparts by replacing Conv blocks with kVNN blocks (with coupled external activations removed accordingly).
Across the denoising settings below, we report reduced-depth kVNN instantiations to study the accuracy--efficiency trade-off; for real-world denoising, we additionally include a same-depth kVNN counterpart to isolate the effect of block replacement.
Three denoising settings are considered:

\noindent\textbf{(1) Fixed-$\sigma$ AWGN denoising.}
We follow the standard AWGN denoising setting and train separate models for each noise level
$\sigma \in \{15,25,50\}$. For each $\sigma$, the CNN baseline (DnCNN\cite{10.1109/TIP.2017.2662206}) and its kVNN counterpart are trained on the same training set and evaluated on BSD68 under the corresponding fixed noise level.

Table~\ref{tab:bsd68_dncnn_fixedsigma} reports the accuracy--efficiency trade-off.
The kVNN variant is a reduced-depth instantiation constructed by replacing Conv blocks with kVNN blocks (with coupled external activations removed accordingly) and decreasing the network depth to improve efficiency. It reduces parameters from 557{,}057 to 372{,}545 (about 33\%) and computation from 2.285 to 1.517 GFLOPs (about 34\%), while maintaining comparable PSNR across all noise levels.

These results suggest that kVNN can achieve similar denoising quality with substantially lower model complexity and computational cost.

\begin{table}[t]
    \centering
    \scriptsize
    \setlength{\tabcolsep}{4pt}
    \resizebox{\columnwidth}{!}{%
    \begin{tabular}{c|cc|cc|cc}
        \toprule
        $\sigma$ &
        \begin{tabular}[c]{@{}c@{}}Params\\ (CNN)\end{tabular} &
        \begin{tabular}[c]{@{}c@{}}Params\\ (kVNN)\end{tabular} &
        \begin{tabular}[c]{@{}c@{}}GFLOPs\\ (CNN)\end{tabular} &
        \begin{tabular}[c]{@{}c@{}}GFLOPs\\ (kVNN)\end{tabular} &
        \begin{tabular}[c]{@{}c@{}}PSNR\\ (CNN)\end{tabular} &
        \begin{tabular}[c]{@{}c@{}}PSNR\\ (kVNN)\end{tabular} \\
        \midrule
        15 & 557{,}057 & 372{,}545 & 2.285 & 1.517 & 31.73 & 31.76 \\
        25 & 557{,}057 & 372{,}545 & 2.285 & 1.517 & 29.23 & 29.27 \\
        50 & 557{,}057 & 372{,}545 & 2.285 & 1.517 & 26.23 & 26.23 \\
        \bottomrule
    \end{tabular}
    } 
    \caption{Fixed-$\sigma$ AWGN denoising on BSD68 with the DnCNN backbone.
    The kVNN variant is constructed by replacing Conv blocks with kVNN blocks (with coupled external activations removed accordingly) and reducing the network depth to improve efficiency.
    We report parameter count, GFLOPs, and PSNR.}

    \label{tab:bsd68_dncnn_fixedsigma}
\end{table}

\noindent\textbf{(2) Random-$\sigma$ AWGN training (single model).}
To evaluate robustness under varying noise levels, we train a single denoiser using AWGN with $\sigma$ randomly sampled per training patch from a uniform distribution $\sigma \sim \mathcal{U}(0,50)$. We report results on Set12 at a representative test noise level $\sigma=25$, and compare CNN-based backbones (DnCNN\cite{10.1109/TIP.2017.2662206}, FFDNet\cite{10762fa0005c48318dd9aef22de37a1c}, and DCANet\cite{10.1007/s00530-024-01469-8}) against their kVNN-based counterparts in terms of PSNR, parameter count, and GFLOPs.

Table~\ref{tab:set12_randomsigma_sigma25} summarizes the results under random-$\sigma$ training. 
Across the three backbones considered (DnCNN\cite{10.1109/TIP.2017.2662206}, 
FFDNet\cite{10762fa0005c48318dd9aef22de37a1c}, and DCANet\cite{10.1007/s00530-024-01469-8}), kVNN matches or improves denoising performance at lower model complexity, yielding a better accuracy--efficiency trade-off.
For example, on DnCNN, kVNN reduces parameters and GFLOPs by about 33\%/34\% and improves PSNR by 0.14 dB (30.31 $\rightarrow$ 30.45 dB). Similar gains are observed for FFDNet (0.14 dB) and DCANet (0.09 dB). All methods are trained and evaluated under the same random-$\sigma$ protocol.

In addition to the quantitative results, Fig.~\ref{fig:set12_vis_randomsigma} provides visual comparisons for
DCANet under the random-$\sigma$ training setting. Compared with the CNN-based DCANet, the kVNN-based DCANet recovers fine textures more faithfully and produces sharper structural boundaries, while suppressing residual noise artifacts.

\begin{table}[t]
    \centering
    \scriptsize
    \setlength{\tabcolsep}{4pt}
    \resizebox{\columnwidth}{!}{%
    \begin{tabular}{l|cc|cc|cc}
        \toprule
        Method &
        \begin{tabular}[c]{@{}c@{}}Params\\ (CNN)\end{tabular} &
        \begin{tabular}[c]{@{}c@{}}Params\\ (kVNN)\end{tabular} &
        \begin{tabular}[c]{@{}c@{}}GFLOPs\\ (CNN)\end{tabular} &
        \begin{tabular}[c]{@{}c@{}}GFLOPs\\ (kVNN)\end{tabular} &
        \begin{tabular}[c]{@{}c@{}}PSNR\\ (CNN)\end{tabular} &
        \begin{tabular}[c]{@{}c@{}}PSNR\\ (kVNN)\end{tabular} \\
        \midrule
        DnCNN\cite{10.1109/TIP.2017.2662206}  & 557{,}057 & 372{,}545 & 2.285 & 1.517 & 30.31 & 30.45 \\
        FFDNet\cite{10762fa0005c48318dd9aef22de37a1c} & 557{,}633 & 373{,}057 & 2.288 & 1.517 & 30.27 & 30.41 \\
        DCANet\cite{10.1007/s00530-024-01469-8} & 1.4M      & 1.3M      & 5.106 & 4.706 & 30.47 & 30.56 \\
        \bottomrule
    \end{tabular}
    } 
    \caption{Random-$\sigma$ AWGN training and fixed-$\sigma$ evaluation on Set12. During training, the noise level is randomly sampled per patch from $\sigma\sim\mathcal{U}(0,50)$ to generate AWGN. We report PSNR on Set12 under AWGN with $\sigma=25$, together with parameter count and GFLOPs. kVNN counterparts are implemented as reduced-depth instantiations by replacing Conv blocks with kVNN blocks (with coupled external activations removed accordingly) and decreasing the network depth to improve efficiency.}

    \label{tab:set12_randomsigma_sigma25}
\end{table}

\begin{figure}[t]
    \centering

    \begin{minipage}{\linewidth}
        \centering
        \includegraphics[width=\linewidth]{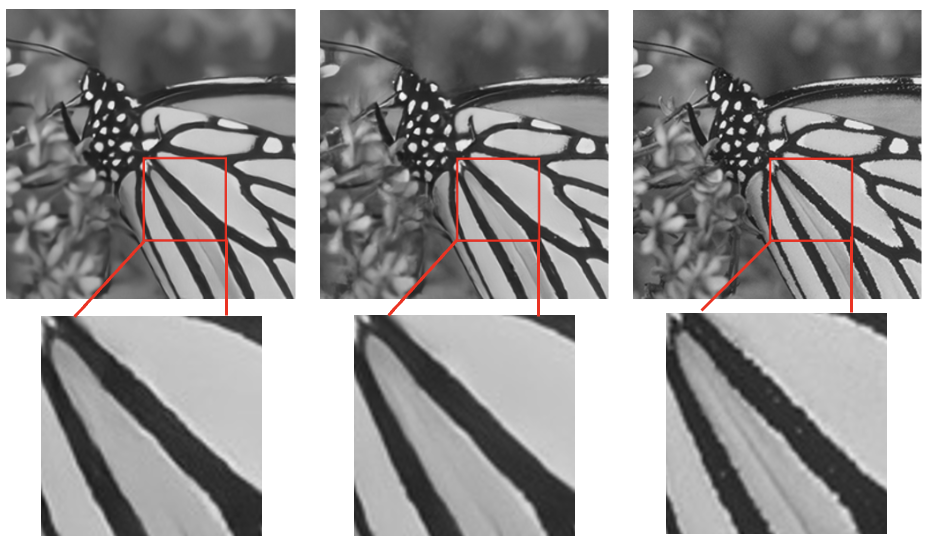}

        \label{fig:set12_vis_butterfly}
    \end{minipage}

    \vspace{0.5em}

    \begin{minipage}{\linewidth}
        \centering
        \includegraphics[width=\linewidth]{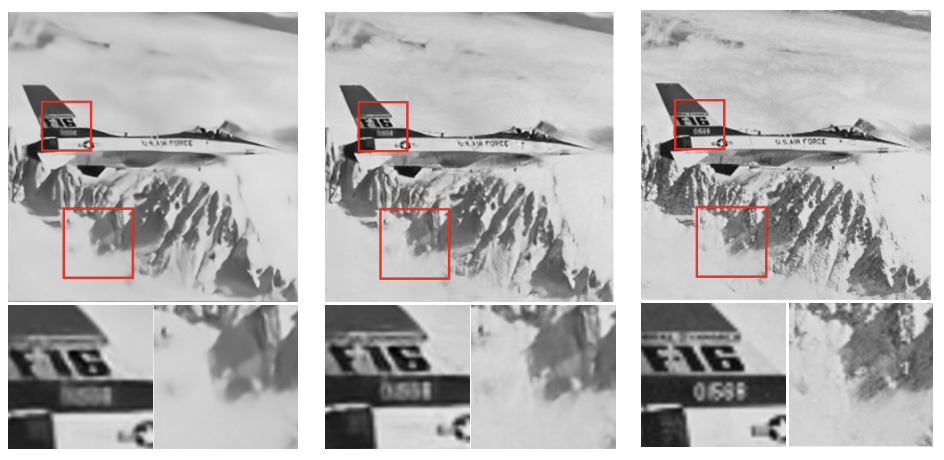}
        \label{fig:set12_vis_plane}
    \end{minipage}

    \caption{Visual comparison of Gaussian denoising reconstruction results on Set12 ($\sigma=25$).
    In each example, the three columns from left to right are: DCANet (CNN-based), DCANet (kVNN-based), and the
    ground-truth image. Zoomed-in patches are provided for clearer comparison of texture recovery and boundary
    sharpness.}
    \label{fig:set12_vis_randomsigma}
\end{figure}

\noindent\textbf{(3) Real-world denoising under a self-supervised framework.}
For real-world sRGB noise removal, we follow the spatially adaptive self-supervised denoising framework in \cite{10203603}, which learns supervision from flat and textured regions using a blind-neighborhood network (BNN\cite{10203603}) and a locally aware network (LAN\cite{10203603}), and then trains a final denoising network (U-Net\cite{10.1007/978-3-319-24574-4_28}) with the learned spatially adaptive supervisions. During inference, BNN and LAN are detached and only the final denoiser (U-Net) is used. We train the denoiser on SIDD Medium and evaluate on the SIDD validation set, reporting results for both the CNN U-Net and the kVNN U-Net variants.

Table~\ref{tab:sidd_medium_sidd_unet} shows that the kVNN-based denoisers remain effective under real-world noise. Under the same self-supervised training recipe, replacing Conv blocks with kVNN blocks improves PSNR (37.29 $\rightarrow$ 37.47 dB), indicating that the proposed operator is compatible with this CNN-tailored pipeline and can yield better restoration quality.
To explicitly assess efficiency, a reduced-depth kVNN variant achieves 37.35 dB with substantially fewer parameters (0.7M vs. 1.0M), demonstrating a favorable accuracy--efficiency trade-off. These results suggest that kVNN is compatible with existing CNN-based training recipes and can be seamlessly plugged into standard architectures to achieve stronger denoising performance.

\begin{table}[t]
    \centering
    \scriptsize
    \setlength{\tabcolsep}{6pt}
    \begin{tabular}{l|c|c}
        \toprule
        Model & Parameters & PSNR \\
        \midrule
        Trained CNN model \cite{10203603}                & 1.0M & 37.29 \\
        Trained kVNN model (Same layers)    & 1.2M & 37.47 \\
        Trained kVNN model (Less layers)    &   0.7M   & 37.35 \\
        \bottomrule
    \end{tabular}
    \caption{Real-world denoising on SIDD Medium. The denoiser uses a U-Net backbone and is trained on the SIDD Medium training set and evaluated on the
    validation set. Same layers: keeps the same U-Net depth while replacing convolution layers with kVNN layers; Less layers: reduces the network depth to illustrate the accuracy--efficiency trade-off.}
    \label{tab:sidd_medium_sidd_unet}
\end{table}

In summary, across both tasks, the results consistently show that kVNN can serve as a replacement for convolutional operators within standard backbones, improving the accuracy--efficiency trade-off under controlled evaluation settings. This conclusion is supported by consistent gains achieved under matched training protocols and comparable evaluation pipelines. It does not rely on introducing task-specific architectural components or bespoke design choices tailored to action recognition or denoising.

\section{Conclusion}
\label{conclusion}

This paper presented a kernelized Volterra Neural Network (kVNN) layer that enables higher-order filtering in a structured and computationally efficient form. By introducing a learnable multi-kernel representation with compact, learnable centers, the proposed formulation provides an order-adaptive parameterization that can be instantiated as a practical network module via branch-wise kVNN filter construction. The kVNN layer can directly replace standard convolutions and be trained end-to-end in existing architectures. Across video action recognition and image denoising, kVNN consistently reduces parameters and GFLOPs while matching or improving performance, even when trained from scratch.

 In future work, we plan to explore extending kVNN to broader operator families and integrating it into transformer-based architectures (e.g., Vision Transformers~\cite{dosovitskiy2021an}). More generally, introducing kVNN-style higher-order interactions into token-mixing modules may offer an efficient way to enrich feature interactions while maintaining strong representation quality.



\bibliographystyle{IEEEbib}   
\bibliography{reference}

@InProceedings{Tran_2015_ICCV,
author = {Tran, Du and Bourdev, Lubomir and Fergus, Rob and Torresani, Lorenzo and Paluri, Manohar},
title = {Learning Spatiotemporal Features With {3D} Convolutional Networks},
booktitle = {Proceedings of the IEEE International Conference on Computer Vision (ICCV)},
month = {December},
year = {2015}
}

@INPROCEEDINGS{8099985,
  author={Carreira, João and Zisserman, Andrew},
  booktitle={2017 IEEE Conference on Computer Vision and Pattern Recognition (CVPR)}, 
  title={Quo Vadis, Action Recognition? A New Model and the Kinetics Dataset}, 
  year={2017},
  volume={},
  number={},
  pages={4724-4733},
  doi={10.1109/CVPR.2017.502}}

@article{10.1145/3422360,
author = {Hu, Hezhen and Zhou, Wengang and Li, Xingze and Yan, Ning and Li, Houqiang},
title = {MV2Flow: Learning Motion Representation for Fast Compressed Video Action Recognition},
year = {2021},
issue_date = {October 2020},
publisher = {Association for Computing Machinery},
address = {New York, NY, USA},
volume = {16},
number = {3s},
issn = {1551-6857},
url = {https://doi.org/10.1145/3422360},
doi = {10.1145/3422360},
journal = {ACM Trans. Multimedia Comput. Commun. Appl.},
month = dec,
articleno = {102},
numpages = {19}
}

@inproceedings{Xiong2021FasterFCoViARFF,
  title={Faster-FCoViAR: Faster Frequency-Domain Compressed Video Action Recognition},
  author={Lu Xiong and Xia Jia and Yue Ming and Jiang Zhou and Fan Feng and Nannan Hu},
  booktitle={British Machine Vision Conference},
  year={2021},
  url={https://api.semanticscholar.org/CorpusID:249893115}
}

@INPROCEEDINGS{9414534,
  author={Yang, Xin and Yang, Chunling},
  booktitle={ICASSP 2021 - 2021 IEEE International Conference on Acoustics, Speech and Signal Processing (ICASSP)}, 
  title={Imrnet: An Iterative Motion Compensation and Residual Reconstruction Network for Video Compressed Sensing}, 
  year={2021},
  volume={},
  number={},
  pages={2350-2354},
  doi={10.1109/ICASSP39728.2021.9414534}}

@article{10.1016/j.neucom.2023.02.045,
author = {Zhang, Junsan and Wang, Xiaomin and Wan, Yao and Wang, Leiquan and Wang, Jian and Yu, Philip S.},
title = {SOR-TC: Self-attentive octave ResNet with temporal consistency for compressed video action recognition},
year = {2023},
issue_date = {May 2023},
publisher = {Elsevier Science Publishers B. V.},
address = {NLD},
volume = {533},
number = {C},
issn = {0925-2312},
url = {https://doi.org/10.1016/j.neucom.2023.02.045},
doi = {10.1016/j.neucom.2023.02.045},
journal = {Neurocomput.},
month = may,
pages = {191–205},
numpages = {15}
}

@inproceedings{10.1007/978-3-030-68763-2_53,
author = {Zhou, Chenghui and Chen, Xiaolei and Sun, Pei and Zhang, Guanwen and Zhou, Wei},
title = {Compressed Video Action Recognition Using Motion Vector Representation},
year = {2021},
isbn = {978-3-030-68762-5},
publisher = {Springer-Verlag},
address = {Berlin, Heidelberg},
url = {https://doi.org/10.1007/978-3-030-68763-2_53},
doi = {10.1007/978-3-030-68763-2_53},
booktitle = {Pattern Recognition. ICPR International Workshops and Challenges: Virtual Event, January 10–15, 2021, Proceedings, Part I},
pages = {701–713},
numpages = {13}
}

@ARTICLE{8249882,
  author={Zhang, Bowen and Wang, Limin and Wang, Zhe and Qiao, Yu and Wang, Hanli},
  journal={IEEE Transactions on Image Processing}, 
  title={Real-Time Action Recognition With Deeply Transferred Motion Vector {CNN}s}, 
  year={2018},
  volume={27},
  number={5},
  pages={2326-2339},
  doi={10.1109/TIP.2018.2791180}}

@INPROCEEDINGS{10096477,
  author={Terao, Hayato and Noguchi, Wataru and Iizuka, Hiroyuki and Yamamoto, Masahito},
  booktitle={ICASSP 2023 - 2023 IEEE International Conference on Acoustics, Speech and Signal Processing (ICASSP)}, 
  title={Efficient Compressed Video Action Recognition Via Late Fusion with a Single Network}, 
  year={2023},
  volume={},
  number={},
  pages={1-5},
  doi={10.1109/ICASSP49357.2023.10096477}}

@INPROCEEDINGS{10376718,
  author={Wasim, Syed Talal and Khattak, Muhammad Uzair and Naseer, Muzammal and Khan, Salman and Shah, Mubarak and Khan, Fahad Shahbaz},
  booktitle={2023 IEEE/CVF International Conference on Computer Vision (ICCV)}, 
  title={Video-FocalNets: Spatio-Temporal Focal Modulation for Video Action Recognition}, 
  year={2023},
  volume={},
  number={},
  pages={13732-13743},
  doi={10.1109/ICCV51070.2023.01267}}

@ARTICLE{11185187,
  author={Ullah, Hayat and Shafique, Muhammad Ali and Khan, Abbas and Munir, Arslan},
  journal={IEEE Transactions on Circuits and Systems for Video Technology}, 
  title={DVFL-Net: A Lightweight Distilled Video Focal Modulation Network for Spatio-Temporal Action Recognition}, 
  year={2025},
  volume={},
  number={},
  pages={1-1},
  doi={10.1109/TCSVT.2025.3615909}}

@INPROCEEDINGS{10203603,
  author={Li, Junyi and Zhang, Zhilu and Liu, Xiaoyu and Feng, Chaoyu and Wang, Xiaotao and Lei, Lei and Zuo, Wangmeng},
  booktitle={2023 IEEE/CVF Conference on Computer Vision and Pattern Recognition (CVPR)}, 
  title={Spatially Adaptive Self-Supervised Learning for Real-World Image Denoising}, 
  year={2023},
  volume={},
  number={},
  pages={9914-9924},
  doi={10.1109/CVPR52729.2023.00956}}

@article{10.1109/TIP.2017.2662206,
author = {Zhang, Kai and Zuo, Wangmeng and Chen, Yunjin and Meng, Deyu and Zhang, Lei},
title = {Beyond a Gaussian Denoiser: Residual Learning of Deep {CNN} for Image Denoising},
year = {2017},
issue_date = {July 2017},
publisher = {IEEE Press},
volume = {26},
number = {7},
issn = {1057-7149},
url = {https://doi.org/10.1109/TIP.2017.2662206},
doi = {10.1109/TIP.2017.2662206},
journal = {Trans. Img. Proc.},
month = jul,
pages = {3142–3155},
numpages = {14}
}

@article{10762fa0005c48318dd9aef22de37a1c,
title = "FFDNet: Toward a fast and flexible solution for {CNN}-Based image denoising",
author = "Kai Zhang and Wangmeng Zuo and Lei Zhang",
year = "2018",
month = sep,
day = "1",
doi = "10.1109/TIP.2018.2839891",
language = "English",
volume = "27",
pages = "4608--4622",
journal = "IEEE Transactions on Image Processing",
issn = "1057-7149",
publisher = "IEEE",
number = "9",
}

@article{10.1007/s00530-024-01469-8,
author = {Wu, Wencong and Lv, Guannan and Duan, Yingying and Liang, Peng and Zhang, Yungang and Xia, Yuelong},
title = {Dual convolutional neural network with attention for image blind denoising},
year = {2024},
issue_date = {Oct 2024},
publisher = {Springer-Verlag},
address = {Berlin, Heidelberg},
volume = {30},
number = {5},
issn = {0942-4962},
url = {https://doi.org/10.1007/s00530-024-01469-8},
doi = {10.1007/s00530-024-01469-8},
journal = {Multimedia Syst.},
month = sep,
numpages = {20}
}

@InProceedings{10.1007/978-3-319-24574-4_28,
author="Ronneberger, Olaf
and Fischer, Philipp
and Brox, Thomas",
editor="Navab, Nassir
and Hornegger, Joachim
and Wells, William M.
and Frangi, Alejandro F.",
title="U-Net: Convolutional Networks for Biomedical Image Segmentation",
booktitle="Medical Image Computing and Computer-Assisted Intervention -- MICCAI 2015",
year="2015",
publisher="Springer International Publishing",
address="Cham",
pages="234--241",
isbn="978-3-319-24574-4"
}

@ARTICLE{6796712,
  author={Franz, Matthias O. and Schölkopf, Bernhard},
  journal={Neural Computation}, 
  title={A Unifying View of Wiener and Volterra Theory and Polynomial Kernel Regression}, 
  year={2006},
  volume={18},
  number={12},
  pages={3097-3118},
  doi={10.1162/neco.2006.18.12.3097}}

@inproceedings{NIPS2015_f3f27a32,
 author = {Alaoui, Ahmed and Mahoney, Michael W},
 booktitle = {Advances in Neural Information Processing Systems},
 editor = {C. Cortes and N. Lawrence and D. Lee and M. Sugiyama and R. Garnett},
 pages = {},
 publisher = {Curran Associates, Inc.},
 title = {Fast Randomized Kernel Ridge Regression with Statistical Guarantees},
 url = {https://proceedings.neurips.cc/paper_files/paper/2015/file/f3f27a324736617f20abbf2ffd806f6d-Paper.pdf},
 volume = {28},
 year = {2015}
}

@article{10.1016/j.inffus.2015.03.001,
author = {Sun, Shiliang and Zhao, Jing and Zhu, Jiang},
title = {A review of Nystr\"{o}m methods for large-scale machine learning},
year = {2015},
issue_date = {November 2015},
publisher = {Elsevier Science Publishers B. V.},
address = {NLD},
volume = {26},
number = {C},
issn = {1566-2535},
url = {https://doi.org/10.1016/j.inffus.2015.03.001},
doi = {10.1016/j.inffus.2015.03.001},
journal = {Inf. Fusion},
month = nov,
pages = {36–48},
numpages = {13}
}

@article{JMLR:v21:19-1031,
  author  = {Linda Chamakh and Emmanuel Gobet and Zolt\'an Szab\'o},
  title   = {Orlicz Random Fourier Features},
  journal = {Journal of Machine Learning Research},
  year    = {2020},
  volume  = {21},
  number  = {145},
  pages   = {1--37},
  url     = {http://jmlr.org/papers/v21/19-1031.html}
}

@inproceedings{NIPS2014_c6cc81e8,
 author = {Dai, Bo and Xie, Bo and He, Niao and Liang, Yingyu and Raj, Anant and Balcan, Maria-Florina and Song, Le},
 booktitle = {Advances in Neural Information Processing Systems},
 editor = {Z. Ghahramani and M. Welling and C. Cortes and N. Lawrence and K.Q. Weinberger},
 pages = {},
 publisher = {Curran Associates, Inc.},
 title = {Scalable Kernel Methods via Doubly Stochastic Gradients},
 url = {https://proceedings.neurips.cc/paper_files/paper/2014/file/c6cc81e8589ebb6accf27b78afad82d9-Paper.pdf},
 volume = {27},
 year = {2014}
}

@InProceedings{pmlr-v119-ghari20a,
  title = 	 {Online Multi-Kernel Learning with Graph-Structured Feedback},
  author =       {Ghari, Pouya M and Shen, Yanning},
  booktitle = 	 {Proceedings of the 37th International Conference on Machine Learning},
  pages = 	 {3474--3483},
  year = 	 {2020},
  editor = 	 {III, Hal Daumé and Singh, Aarti},
  volume = 	 {119},
  series = 	 {Proceedings of Machine Learning Research},
  month = 	 {13--18 Jul},
  publisher =    {PMLR},
  url = 	 {https://proceedings.mlr.press/v119/ghari20a.html}
}

@ARTICLE{8627941,
  author={Liu, Xinwang and Wang, Lei and Zhu, Xinzhong and Li, Miaomiao and Zhu, En and Liu, Tongliang and Liu, Li and Dou, Yong and Yin, Jianping},
  journal={IEEE Transactions on Pattern Analysis and Machine Intelligence}, 
  title={Absent Multiple Kernel Learning Algorithms}, 
  year={2020},
  volume={42},
  number={6},
  pages={1303-1316},
  doi={10.1109/TPAMI.2019.2895608}}

@INPROCEEDINGS{7410527,
  author={Lin, Tsung-Yu and RoyChowdhury, Aruni and Maji, Subhransu},
  booktitle={2015 IEEE International Conference on Computer Vision (ICCV)}, 
  title={Bilinear {CNN} Models for Fine-Grained Visual Recognition}, 
  year={2015},
  volume={},
  number={},
  pages={1449-1457},
  doi={10.1109/ICCV.2015.170}}

@INPROCEEDINGS{7780410,
  author={Gao, Yang and Beijbom, Oscar and Zhang, Ning and Darrell, Trevor},
  booktitle={2016 IEEE Conference on Computer Vision and Pattern Recognition (CVPR)}, 
  title={Compact Bilinear Pooling}, 
  year={2016},
  volume={},
  number={},
  pages={317-326},
  doi={10.1109/CVPR.2016.41}}

@article{10.5555/3455716.3455839,
author = {Kossaifi, Jean and Lipton, Zachary C. and Kolbeinsson, Arinbj\"{o}rn and Khanna, Aran and Furlanello, Tommaso and Anandkumar, Anima},
title = {Tensor regression networks},
year = {2020},
issue_date = {January 2020},
publisher = {JMLR.org},
volume = {21},
number = {1},
issn = {1532-4435},
journal = {J. Mach. Learn. Res.},
month = jan,
articleno = {123},
numpages = {21}
}

@INPROCEEDINGS{9156685,
  author={Chrysos, Grigorios G. and Moschoglou, Stylianos and Bouritsas, Giorgos and Panagakis, Yannis and Deng, Jiankang and Zafeiriou, Stefanos},
  booktitle={2020 IEEE/CVF Conference on Computer Vision and Pattern Recognition (CVPR)}, 
  title={P–nets: Deep Polynomial Neural Networks}, 
  year={2020},
  volume={},
  number={},
  pages={7323-7333},
  doi={10.1109/CVPR42600.2020.00735}}

@ARTICLE{9247263,
author={Banerjee, Monami and Chakraborty, Rudrasis and Bouza, Jose and Vemuri, Baba C.},
journal={ IEEE Transactions on Pattern Analysis \& Machine Intelligence },
title={{ VolterraNet: A Higher Order Convolutional Network With Group Equivariance for Homogeneous Manifolds }},
year={2022},
volume={44},
number={02},
ISSN={1939-3539},
pages={823-833},
doi={10.1109/TPAMI.2020.3035130},
url = {https://doi.ieeecomputersociety.org/10.1109/TPAMI.2020.3035130},
publisher={IEEE Computer Society},
address={Los Alamitos, CA, USA},
month=feb}

@inproceedings{dosovitskiy2021an,
  author = {Dosovitskiy, Alexey and Beyer, Lucas and Kolesnikov, Alexander and Weissenborn, Dirk and Zhai, Xiaohua and Unterthiner, Thomas and Dehghani, Mostafa and Minderer, Matthias and Heigold, Georg and Gelly, Sylvain and Uszkoreit, Jakob and Houlsby, Neil},
  booktitle = {International Conference on Learning Representations},
  title = {An Image is Worth 16x16 Words: Transformers for Image Recognition at Scale},
  url = {https://openreview.net/forum?id=YicbFdNTTy},
  year = 2021
}

@article{DBLP:journals/corr/abs-1212-0402,
  author       = {Khurram Soomro and
                  Amir Roshan Zamir and
                  Mubarak Shah},
  title        = {{UCF101:} {A} Dataset of 101 Human Actions Classes From Videos in
                  The Wild},
  journal      = {CoRR},
  volume       = {abs/1212.0402},
  year         = {2012},
  url          = {http://arxiv.org/abs/1212.0402},
  eprinttype    = {arXiv},
  eprint       = {1212.0402},
  bibsource    = {dblp computer science bibliography, https://dblp.org}
}

@INPROCEEDINGS{6126543,
  author={Kuehne, H. and Jhuang, H. and Garrote, E. and Poggio, T. and Serre, T.},
  booktitle={2011 International Conference on Computer Vision}, 
  title={HMDB: A large video database for human motion recognition}, 
  year={2011},
  volume={},
  number={},
  pages={2556-2563},
  doi={10.1109/ICCV.2011.6126543}}

@inproceedings{10.5555/648300.755324,
author = {Sch\"{o}lkopf, Bernhard and Herbrich, Ralf and Smola, Alex J.},
title = {A Generalized Representer Theorem},
year = {2001},
isbn = {3540423435},
publisher = {Springer-Verlag},
address = {Berlin, Heidelberg},
booktitle = {Proceedings of the 14th Annual Conference on Computational Learning Theory and 5th European Conference on Computational Learning Theory},
pages = {416–426},
numpages = {11},
series = {COLT '01/EuroCOLT '01}
}

@article{JMLR:v6:drineas05a,
  author  = {Petros Drineas and Michael W. Mahoney},
  title   = {On the Nystrom Method for Approximating a Gram Matrix for Improved Kernel-Based Learning},
  journal = {Journal of Machine Learning Research},
  year    = {2005},
  volume  = {6},
  number  = {72},
  pages   = {2153--2175},
  url     = {http://jmlr.org/papers/v6/drineas05a.html}
}

@book{10.7551/mitpress/4175.001.0001,
    author = {Schölkopf, Bernhard and Smola, Alexander J.},
    title = {Learning with Kernels: Support Vector Machines, Regularization, Optimization, and Beyond},
    publisher = {The MIT Press},
    year = {2001},
    month = {Dec},
    isbn = {9780262256933},
    doi = {10.7551/mitpress/4175.001.0001},
    url = {https://doi.org/10.7551/mitpress/4175.001.0001},
}

@article{JMLR:v25:21-1082,
  author  = {Siddharth Roheda and Hamid Krim and Bo Jiang},
  title   = {Volterra Neural Networks ({VNN}s)},
  journal = {Journal of Machine Learning Research},
  year    = {2024},
  volume  = {25},
  number  = {182},
  pages   = {1--29},
  url     = {http://jmlr.org/papers/v25/21-1082.html}
}

@article{Roheda_Krim_2020, title={Conquering the {CNN} Over-Parameterization Dilemma: A Volterra Filtering Approach for Action Recognition}, volume={34}, url={https://ojs.aaai.org/index.php/AAAI/article/view/6870}, DOI={10.1609/aaai.v34i07.6870}, abstractNote={&lt;p&gt;The importance of inference in Machine Learning (ML) has led to an explosive number of different proposals in ML, and particularly in Deep Learning. In an attempt to reduce the complexity of Convolutional Neural Networks, we propose a Volterra filter-inspired Network architecture. This architecture introduces controlled non-linearities in the form of interactions between the delayed input samples of data. We propose a cascaded implementation of Volterra Filtering so as to significantly reduce the number of parameters required to carry out the same classification task as that of a conventional Neural Network. We demonstrate an efficient parallel implementation of this Volterra Neural Network (VNN), along with its remarkable performance while retaining a relatively simpler and potentially more tractable structure. Furthermore, we show a rather sophisticated adaptation of this network to nonlinearly fuse the RGB (spatial) information and the Optical Flow (temporal) information of a video sequence for action recognition. The proposed approach is evaluated on UCF-101 and HMDB-51 datasets for action recognition, and is shown to outperform state of the art CNN approaches.&lt;/p&gt;}, number={07}, journal={Proceedings of the AAAI Conference on Artificial Intelligence}, author={Roheda, Siddharth and Krim, Hamid}, year={2020}, month={Apr.}, pages={11948-11956} }

@article{GHANEM2023200210,
  title={Latent Code-Based Fusion: A Volterra Neural Network Approach},
  author={Sally Ghanem and Siddharth Roheda and Hamid Krim},
  journal={Intell. Syst. Appl.},
  year={2021},
  volume={18},
  pages={200210},
  url={https://api.semanticscholar.org/CorpusID:233209691}
}

\clearpage
\onecolumn
\appendix

\section{Theoretical Properties of the Learnable Multi-Kernel Construction}
\label{app:theory_mk}

This appendix collects theoretical results that support the proposed learnable multi-kernel representation.

\begin{theorem}[Feature map representation]\label{thm:map}
Let $K:\mathbb{R}^d\times\mathbb{R}^d\to\mathbb{R}$ be a symmetric and positive semidefinite kernel. Then there exists a Hilbert space $\mathcal{F}$ and a map $\phi:\mathbb{R}^d\to\mathcal{F}$ such that
\begin{equation}
\label{eq:rkhs_featuremap_bg_appendix}
K(x,x')=\langle \phi(x),\phi(x')\rangle_{\mathcal{F}},\qquad \forall x,x'\in\mathbb{R}^d,
\end{equation}
where $\langle \cdot , \cdot \rangle_{\mathcal{F}}$ is the inner product in $\mathcal{F}$.
\end{theorem}

\noindent\textit{This is a standard result in RKHS theory; see, e.g., \cite{10.7551/mitpress/4175.001.0001}.}

Subsequently, we adopt the following terminology: the function $K$ is called the reproducing kernel, $K(\cdot, x)$ is referred to as a kernel atom (or representer), $\phi$ is the feature map induced by $K$, and $\mathcal{F}$ is the resulting reproducing kernel Hilbert space (RKHS).

\begin{lemma}[Polynomial kernel]\label{thm:kernel}
For any $r\in\mathbb{N}$, the kernel
\begin{equation}
\label{eq:poly_kernel_bg_appndix}
K_r(x,x') := (x^\top x')^{r}
\end{equation}
is positive semidefinite and admits a finite-dimensional feature map $\phi_r$ satisfying
$K_r(x,x')=\langle \phi_r(x),\phi_r(x')\rangle$.
\end{lemma}

\begin{proof}
For any $r\in\mathbb{N}$ and $x, x'\in\mathbb{R}^d$, the multinomial expansion gives:
\[
(x^\top x')^{r}=\Big(\sum_{k=1}^d x_k x'_k\Big)^r
=\sum_{\alpha\in\mathbb{N}^d:\ |\alpha|=r}\binom{r}{\alpha}\, x^\alpha (x')^\alpha,
\]
where $\alpha =(\alpha_1, \dots, \alpha_d)$ is a multi-index with $|\alpha| \triangleq \sum_{j=1}^d \alpha_j = r$, and the multinomial coefficient is given by
$\binom{r}{\alpha} \triangleq \frac{r!}{\alpha_1! \alpha_2! \cdots \alpha_d!}$.
The monomial term is denoted by $x^{\alpha} \triangleq x_1^{\alpha_1 } \cdots x_d^{\alpha_d}$.

Let the feature map $\phi_r(x)$ be a vector indexed by all $\alpha$ such that $|\alpha|=r$, with the component corresponding to $\alpha$ defined as
$\big(\phi_r(x)\big)_\alpha:=\sqrt{\binom{r}{\alpha}}\,x^\alpha$.
Then,
\[
\langle \phi_r(x),\phi_r(x')\rangle
=\sum_{|\alpha|=r}\binom{r}{\alpha}\, x^\alpha (x')^\alpha
=(x^\top x')^{r}=K_r(x,x').
\]
Therefore $K_r$ is positive semidefinite.
\end{proof}

\begin{theorem}[Multi-kernel validity and order-wise decomposition]
\label{lem:mk_compact_nostep}
Let $K_r(x,x')\triangleq (x^\top x')^r$ for $r=1,\ldots,p$, and let $a_r\ge 0$.
Define the multi-kernel
\begin{equation}
\label{eq:Ksum_compact_nostep}
K(x,x')\triangleq \sum_{r=1}^{p} a_r^2 K_r(x,x')
=\sum_{r=1}^{p} a_r^2 (x^\top x')^r .
\end{equation}
Then $K$ is positive semidefinite. Moreover, if $\mathcal{H}(K_r)$ denotes the RKHS induced by $K_r$, any $f\in\mathcal{H}(K)$ admits a decomposition $f=\sum_{r=1}^{p} f_r$ with $f_r\in\mathcal{H}(K_r)$.
\end{theorem}

\begin{proof}
For each $r$, by Lemma~\ref{thm:kernel}, the kernel $K_r$ admits a feature map $\phi_r:\mathbb{R}^d\to\mathcal{F}_r$.
Consider the direct-sum feature space $\mathcal{F}\triangleq \mathcal{F}_1\oplus\cdots\oplus\mathcal{F}_p$ and define the concatenated feature map $\phi(x)\in \mathcal{F}$ as:
\begin{equation}
\label{eq:feat_concat_nostep}
\phi(x)\triangleq \big(a_1\phi_1(x),\,\ldots,\,a_p\phi_p(x)\big).
\end{equation}
The kernel induced by $\phi$ is
\[
\langle \phi(x),\phi(x')\rangle_{\mathcal{F}}
=\sum_{r=1}^{p} a_r^2 \langle \phi_r(x),\phi_r(x')\rangle_{\mathcal{F}_r}
=K(x,x'),
\]
which proves that $K$ is positive semidefinite.

Let $f\in\mathcal{H}(K)$. By the feature-space representation of the RKHS induced by $\phi$, there exists a vector
$w = (w_1, \dots, w_p) \in \mathcal{F}$ such that $f(x)=\langle w,\phi(x)\rangle_{\mathcal{F}}$.
Substituting~\eqref{eq:feat_concat_nostep} yields
\[
f(x)=\sum_{r=1}^{p} \langle w_r, a_r\phi_r(x)\rangle_{\mathcal{F}_r}
=\sum_{r=1}^{p} f_r(x),
\]
where we define $f_r(x)\triangleq \langle a_r w_r,\phi_r(x)\rangle_{\mathcal{F}_r}$. Since $\phi_r$ is the feature map for $K_r$, it follows that $f_r\in\mathcal{H}(K_r)$.
\end{proof}

\begin{definition}
The function $f:\mathbb{R}^d\to\mathbb{R}$ defined above, ${\displaystyle f(x)=\sum_{r=1}^{p} f_r(x)}$, is called the $p$-th order truncated Volterra mapping.
\end{definition}
\begin{definition}
A homogeneous polynomial of degree $r$ is a polynomial
$p(x)=p(x_1,\cdots,x_d)$
such that $$p(\lambda x)=\lambda^r p(x), \quad \forall \lambda \in \mathbb{R}, \ x \in \mathbb{R}^d.$$
Specifically, every monomial term has total degree exactly $r$.
That is, each term of the homogeneous polynomial takes the form:
$$c x^{\alpha} =c x_1^{\alpha_1} x_2^{\alpha_2}\cdots x_d^{\alpha_d}$$
with $ |\alpha|=\alpha_1+\alpha_2+\cdots+\alpha_d=r$.
\end{definition}

\begin{theorem}[Finite atomic representation in polynomial linear spaces]
\label{thm:finite_atomic_rep}
Let $f:\mathbb{R}^d\to\mathbb{R}$ be a $p$-th order truncated Volterra mapping. Viewing the space of degree-$r$ homogeneous polynomials as a finite-dimensional linear space, each component $f_r$ admits an exact representation using a finite set of kernel atoms. Specifically, for each order $r\in\{1,\ldots,p\}$, there exists a finite integer $M_r$ along with parameters $\{\gamma_{r,i}\}_{i=1}^{M_r}\subset\mathbb{R}$ and $\{w_{r,i}\}_{i=1}^{M_r}\subset\mathbb{R}^d$ such that
\begin{equation}
\label{eq:fr_span}
f_r(x)=\sum_{i=1}^{M_r}\gamma_{r,i}\,(x^\top w_{r,i})^r .
\end{equation}
Consequently, the entire function $f$ can be exactly represented as a multi-kernel expansion:
\begin{equation}
\label{eq:f_total_span}
f(x)=\sum_{r=1}^{p}\sum_{i=1}^{M_r}\gamma_{r,i}\,(x^\top w_{r,i})^r .
\end{equation}
\end{theorem}

\begin{proof}
The proof proceeds by analyzing the linear structure of the function space. Let $\mathcal{V}_r$ denote the vector space of all homogeneous polynomials of degree $r$ on $\mathbb{R}^d$. It is well-established that $\mathcal{V}_r$ is a linear space of finite dimension.

Consider the family of kernel atoms defined by the power functions $\phi_w(x) \triangleq (x^\top w)^r$, parameterized by $w \in \mathbb{R}^d$. Each $\phi_w$ is an element of $\mathcal{V}_r$. Moreover, by the multinomial expansion,
\[
(x^\top w)^r = \Big(\sum_{j=1}^{d} w_j x_j\Big)^r
\]
is a degree-$r$ homogeneous polynomial in $x$.

This family of atoms spans $\mathcal{V}_r$. Indeed, identify $\mathcal{V}_r$ with the space of symmetric $r$-tensors on $\mathbb{R}^d$; then $(x^\top w)^r$ corresponds to the rank-one symmetric tensor $w^{\otimes r}$, and rank-one tensors span the full tensor space. Hence $\mathrm{span}\{\phi_w:w\in\mathbb{R}^d\}=\mathcal{V}_r$.
Since $\mathcal{V}_r$ is finite-dimensional, there exist $w_{r,1},\ldots,w_{r,M_r}\in\mathbb{R}^d$ such that
$\mathrm{span}\{(x^\top w_{r,i})^r\}_{i=1}^{M_r}=\mathcal{V}_r$.
Therefore, for any $f_r\in\mathcal{V}_r$, there exist coefficients $\{\gamma_{r,i}\}_{i=1}^{M_r}$ such that \eqref{eq:fr_span} holds.
Summing these components over $r=1,\ldots,p$ yields \eqref{eq:f_total_span}.
\end{proof}

\begin{remark}
This theorem provides the theoretical guarantee that our learnable-kernel formulation is sufficiently expressive to represent any \emph{truncated} (finite-order) Volterra mapping using a finite number of parameters. By treating the centers $\{w_{r,i}\}$ as learnable parameters, the model can adaptively discover a compact set of atoms to represent the target function without relying on a pre-fixed basis.
\end{remark}

\end{document}